\documentclass{article}
\usepackage{arxiv}

\newif\ifacm
\acmfalse

\AtBeginDocument{%
  }

\ifacm
    \setcopyright{acmcopyright}
    \copyrightyear{2022}
    \acmYear{2022}
    \acmDOI{XXXXXXX.XXXXXXX}
    
    \acmConference[WSDM '23]{ACM WSDM}{Feb 03--05,
      2023}{Singapore}
    \acmPrice{15.00}
    \acmISBN{978-1-4503-XXXX-X/18/06}
\fi

\usepackage[utf8]{inputenc} 
\usepackage[T1]{fontenc}    
\usepackage{hyperref}       
\usepackage{url}            
\usepackage{booktabs}       
\usepackage{amsfonts}       
\usepackage{nicefrac}       
\usepackage{microtype}      
\usepackage{xcolor}         

\usepackage{natbib}
\setcitestyle{numbers,open={[},close={]}}
\usepackage{graphicx}
\usepackage{algpseudocode}
\usepackage{comment}
\usepackage{xspace}
\usepackage{amsmath,amsthm}
\usepackage{bm}
\newtheorem{definition}[]{Definition}
\newtheorem{proposition}{Proposition}

\usepackage{enumitem}
\setitemize{noitemsep,topsep=0pt,parsep=0pt,partopsep=0pt,leftmargin=*}
\setenumerate{noitemsep,topsep=0pt,parsep=0pt,partopsep=0pt,leftmargin=*}

\newcommand{\vecV}{\bm{V}}
\newcommand{\vecv}{\bm{v}}
\newcommand{\vecW}{\bm{W}}
\newcommand{\vecw}{\bm{w}}
\newcommand{\vecU}{\bm{U}}
\newcommand{\vecu}{\bm{u}}
\newcommand{\vecF}{\bm{F}}
\newcommand{\vecZ}{\bm{Z}}
\newcommand{\vecz}{\bm{z}}
\newcommand{\vecS}{\bm{S}}
\newcommand{\vecs}{\bm{s}}
\newcommand{\vecX}{\bm{X}}
\newcommand{\vecx}{\bm{x}}

\newcommand{\actcause}{{\textit{CF-Shapley}}\xspace}
\newcommand{\ad}{{ad demand}\xspace}
\newcommand{\qv}{{query volume}\xspace}
\newcommand{\vol}{{qv}}
\newcommand{\dem}{{ad}}
\newcommand{\den}{{den}}
\newcommand{\supdelta}{{\texttt{QVolumeDelta}}\xspace}
\newcommand{\demdelta}{{\texttt{AdDemandDelta}}\xspace}
\newcommand{\proddelta}{{\texttt{ProductDelta}}\xspace}
\newcommand{\stdshapley}{{\texttt{Shapley}}\xspace}
\newcommand{\doshapley}{{\texttt{DoShapley}}\xspace}

\newtheorem*{axiom*}{Axioms}

\ifacm
\else
\author{%
Amit Sharma \\
Microsoft Research \\
\texttt{amshar@microsoft.com} \\
\And
Hua Li \\
Microsoft Bing Ads \\
\texttt{huli@microsoft.com}\\
\And
Jian Jiao \\
Microsoft Bing Ads \\
\texttt{jian.jiao@microsoft.com}
}

\fi
\begin{document}

\title{The Counterfactual-Shapley Value: Attributing Change in System Metrics}

\ifacm
\else
\maketitle
\fi

\begin{abstract} 
Given an unexpected change in the output metric 
of a large-scale system, it is important to answer \textit{why} the change occurred: which inputs caused the change in metric? A key component  of such an attribution question is estimating the \textit{counterfactual}: the  (hypothetical) change in the system metric due to a specified change in a single input.  However, due to inherent stochasticity and complex interactions between parts of the system, it is difficult to model an output metric directly. 
We utilize the computational structure of a system to break up the modelling task into sub-parts, such that each sub-part corresponds to a more stable mechanism that can be modelled accurately over time. Using the system's structure also helps to view the metric as a computation over a structural causal model (SCM),  thus providing a principled way to estimate counterfactuals. Specifically, we propose a method to estimate counterfactuals 
using time-series predictive models and construct an attribution score, \actcause, that is consistent with desirable axioms for attributing an observed change in the output metric. Unlike past work on causal shapley values, our proposed method
can attribute a single observed change in output (rather than a population-level effect) and thus provides more accurate attribution scores when evaluated on simulated datasets.  As a real-world application, we  analyze 
  a query-ad matching system with the goal of attributing observed change in a metric for  ad matching density. 
  Attribution scores explain how \qv and \ad  from different query categories affect the ad matching density, leading to actionable insights and
  uncovering the role of external events (e.g., ``Cheetah Day'') in driving the matching density.
\end{abstract}




\section{Introduction}
In large-scale systems, a common problem is to explain the reasons for a change in the output, especially for unexpected and big changes. Explaining the reasons or attributing the change to input factors can help isolate the cause and debug it if the change is undesirable, or suggest ways to amplify the change if desirable. For example, in a distributed system, system failure~\cite{zhang2019inflection} or performance anomaly~\cite{nagaraj2012structured,attariyan2012xray} are important undesirable outcomes.  In online platforms such as e-commerce websites or search websites, a desirable outcome is increase in revenue and it is important to understand why the revenue increased or  decreased~\cite{singal2022shapley,dalessandro2012causally}. 

Technically, this problem can be framed as an attribution problem~\cite{efron2020prediction,yamamoto2012understanding, dalessandro2012causally}. Given a set of candidate factors, which of them can best explain the observed change in output? Methods include statistical analysis based on conditional probabilities~\cite{berman2018beyondmta,ji2016probabilistic-mta,shao2011mta} or computation of game-theoretic attribution scores like Shapley value~\cite{lundberg2017unified,singal2022shapley,dalessandro2012causally}. However, most past work assumes that  the output can be written as a function of the inputs, ignoring any structure in the computation of the output. 

\begin{figure}
    \centering
    \includegraphics[width=0.8\columnwidth]{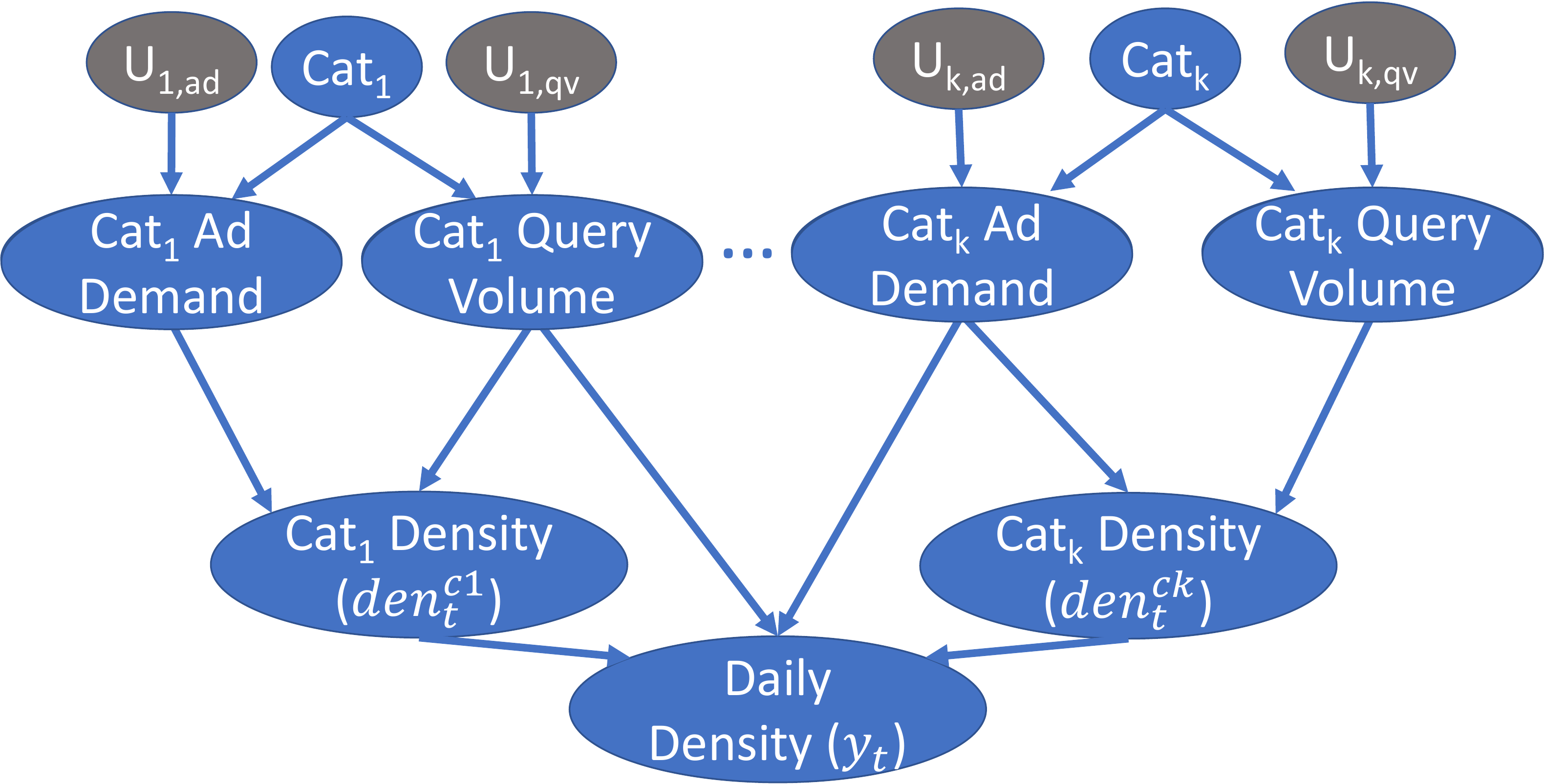}
    \caption{Causal graph for an ad matching system that reflects computation of the matching density metric. For each query category, the number of queries (query volume) and ads (ad demand) determine the categorical density for each day. Different categorical densities combine to yield the daily density. Goal is to attribute daily density to ad demand and query volume of different categories. The category density is directly affected by category-wise ad demand and query volume  and thus has a relatively more stable relationship with the inputs than overall daily density. 
     }
    \label{fig:causal-graph}
    \vspace{-3em}
\end{figure}

In this paper, we consider large-scale systems such as search or ad systems where output metrics are aggregated over different kinds of inputs or composed over multiple pipeline stages,  leading to a natural computational structure (instead of a single function of the inputs). For example, in an ad system, the number of ads that are matched per query is a composite measure that is composed of an analogous metric over each query category (see Figure~\ref{fig:causal-graph}). While the overall matching density may fluctuate, the matching density per category is expected to be more stably associated with the input queries and ads. As another example, the output metric may be a result of a series of modules in a pipeline, e.g., recommendations that are finally shown to a user may be a result of multiple  pipeline stages where each stage filters some items. Our key insight is that utilizing the computational structure of a real-world system can break up the system into smaller sub-parts that stay stable over time and thus  can be modelled accurately. In other words, the system's computation can be modelled as a set of independent, causal mechanisms~\cite{peters2017elements} over a structural causal model (SCM)~\cite{pearl2009book}. 

Modeling the system's computation as a SCM also provides a principled way to define an  attribution score. Specifically, we show that attribution can be defined in terms of \textit{counterfactuals} on the SCM. Following recent work on causal shapley values~\cite{heskes2020causal,pmlr-v162-jung22a}, we posit four axioms that any desirable attribution method for an output metric should satisfy. We then propose a counterfactual variant of the Shapley value that satisfies all these properties. Thus, given the computational structure, our proposed \actcause method has the following steps: \textbf{1)}  utilize machine learning algorithms to  fit the SCM and compute counterfactual values of the metric under any input,  and \textbf{2) } use the estimated counterfactuals to construct an attribution score to rank the contribution of different inputs. On simulated data, our results  show that the proposed method is significantly more accurate for explaining inputs' contribution to an observed change in a system metric, compared to Shapley value~\cite{lundberg2017unified} or its recent causal variants~\cite{heskes2020causal,pmlr-v162-jung22a}. 


We apply the proposed method, \actcause attribution, to a large-scale ad matching system that outputs  relevant ads for each search query issued by a user. The key outcome is \textit{matching density}, the number of ads matched per query. This density is roughly proportional to revenue generated, since only the queries for which ads are selected contribute to revenue. There are two main causes for a change in matching density: change in query volume or change in demand from advertisers. Given that queries are typically organized by categories, the attribution problem is to understand which of these two are driving an observed change in matching density, and from which categories. 

To do so, we construct a causal graph representing the system's computation pipeline (Figure~\ref{fig:causal-graph}).  Given six months of  system's log data, we repurpose time-series prediction models to learn the structural equation for category-wise density as a function of query volume and ad demand, its parents in the graph. 
For this system, we find that category-wise attribution is possible with minimal assumptions, while attribution between query volume and ad demand requires knowledge of the structural equations that generate category-wise density. 
In both cases, we show how the \actcause method can be used to estimate the system's  counterfactual outputs and the resultant attribution scores. 
As a sanity check,  \actcause attribution scores satisfy the \textit{efficiency} property for attributing the matching density metric: their sum matches the observed change in density. We then use \actcause scores to explain density changes on five outlier days from November to December 2021, uncovering insights on how changes in \qv or \ad  for different categories affects the density metric. We validate the results through an analysis of external events during the time period.

To summarize, our contributions include, 
\begin{itemize}
    \item A method for attributing  metrics in a large-scale system utilizing its computational structure as a causal graph, that outperforms recent Shapley value-based methods on simulated data.
    \item A case study on estimating counterfactuals in a real-world ad matching system, providing a principled way for attributing change in its output metric.
\end{itemize}

\section{Related Work} 
Our work considers a causal interpretation of the attribution problem. Unlike attribution methods on predictions from a (deterministic) machine learning model~\cite{lundberg2017unified,ide2021anomaly}, here we are interested in attributing real-world outcomes where the data-generating process includes noise.  Since the attribution problem concerns explaining a single outcome or event, we focus on causality on individual outcomes~\cite{halpern2016book} rather than \textit{general} causality that deals with the average effect of a cause on the outcome over a (sub)population~\cite{pearl2009book}. In other words, we are interested in estimating the counterfactual, given that we already have an observed event. Counterfactuals are the hardest problem in Pearl's ladder of causation, compared to observation and intervention~\cite{pearl2019seven}.  

While counterfactuals have been applied in feature attribution for machine learning models~\cite{kommiya2021towards,verma2020counterfactual}, less work has been done for attributing real-world outcomes in systems using formal counterfactuals. Recent work uses the \textit{do}-intervention to propose do-shapley values~\cite{heskes2020causal,pmlr-v162-jung22a} that attribute the interventional quantity $P(Y|do(\vecV))$ across different inputs $v \in \vecV$. While do-shapley values are useful for calculating the average effect of different inputs on the output $Y$, they are not applicable for attributing an \textit{individual} change in the  output. For attributing individual changes,  \cite{janzing2019causaloutliers} analyze root cause identification for outliers in a structural causal model, and find that attribution conditional on the parents of a node is more effective than global attribution. They quantify the attribution using information theoretic scores, but do not provide any axiomatic characterization of the resulting attribution score.
In this work, we propose four axioms that characterize desirable properties for an attribution score for explaining individual change in output and present the CF-Shapley value that satisfies those axioms.

\noindent \textbf{Attribution in ad systems. }
Multi-touch attribution is the most common attribution problem studied in online ad systems. Given an ad click, the goal is to assign credit to the different preceding exposures of the same item to the user, e.g., previous ad exposures, emails, or other media. Multiple methods have been proposed to estimate the attribution such as attributing all to the last exposure~\cite{berman2018beyondmta}, an average over all exposures, or using probabilistic models to model the click data as a function of the input exposures~\cite{shao2011mta,ji2016probabilistic-mta}. Recent methods utilize the game-theoretic attribution score using Shapley values that summarizes the attribution over multiple simulations of input variables, with~\cite{dalessandro2012causally} or without a causal interpretation~\cite{singal2022shapley}. 
Multi-touch attribution can be considered as a one-level SCM problem, where there is an output node being affected by all input nodes. It does not cover more complex systems where there is a computational structure.

\noindent \textbf{Performance Anomaly Attribution. }
Computational structure (e.g., specific system capabilities or logs) has been considered in the systems literature to root-cause performance anomalies~\cite{attariyan2012xray} or system failures~\cite{zhang2019inflection}. Some methods use causal reasoning to motivate their attribution algorithm, but they do so informally. 
Our work provides a formal analysis of the system attribution problem.

\section{Defining the attribution problem}
For a system's outcome metric $Y$, let $Y=y_t$ be a value that needs to be explained (e.g., an extreme value). Our goal is to explain the value by \textit{attributing} it to a set of input variables, $\vecX$.  Can we rank the  variables by their contribution in \textit{causing} the outcome? 

For example, consider a system that crashes whenever its load crosses 0.9 units. The system's crash metric can be described by the following structural equations, $Y=I_{Load>=0.9}; Load=0.5X_1+ 0.4X_2+0.9X_3; X_i= Bernoulli(0.5) \forall i$. The corresponding graph for the system has the following edges: ${X_1, X_2, X_3} \rightarrow Load; Load \rightarrow Y$. The value of each input $X_i$ is affected by the independent error terms through the Bernoulli distribution.  Suppose the initial reference value was $(X_1=0,X_2=0, X_3=0, Y=0)$ and the next observed value is $(X_1=1,X_2=1, X_3=1, Y=1)$. Given that the system crashed ($Y=1$), how do we attribute it to $X_1, X_2, X_3$?  
Intuitively, $X_3$ is a sufficient cause of the crash since changing $X_3=1$ would lead to the crash irrespective of values of other variables. However,  $X_1$ and $X_2$ can be equally a reason for this particular crash since their coefficients sum to $0.9$. 
However, if either of $X_1$ or $X_2$ are observed to be zero, then the other one  cannot explain the crash. 
This example indicates that the attribution for any input variable depends on the equations of the data-generating process and also on the values of other variables. 
\subsection{Attribution score for system metric change}
We now define the \textit{attribution score} for explaining an observed value wrt a reference value. While system inputs can be continuous, we utilize the fact that system metrics are measured and compared over time. That is, we are often interested in attribution for a metric value compared to an reference timestamp. Reference values are typically chosen from previous values that are expected to be comparable (E.g., metric value at last hour or last week).  
By comparing to a reference timestamp, we simplify the problem by considering only two values of a continuous variable:
its \textit{observed} value, and its value on the\textit{ reference} timestamp. 

Formally, we express the problem of attribution of an outcome metric, $Y=y_t$ as explaining change in the metric wrt. a reference, $\Delta Y = y_t - y'$: Why did the outcome value change from $y'$ to $y_t$? 

\begin{definition}\label{def:attr-problem}
\textbf{Attribution Score. }
Let $Y=y_t$ and $Y=y'$ be the observed and reference values respectively of a system metric. Let $\vecV$ be the set of   input variables. Then, an \textit{attribution score} for $X \in \vecV$ provides the contribution of  $X$ in causing the change from $y'$ to $y_t$. 
\end{definition}

\subsection{The need for SCM and counterfactuals}
To estimate the \textit{causal} contribution, we  need to model the
data-generating process from  input variables to  the outcome. This is usually done by a structural causal model (SCM) $M$, that consists of a causal graph and structural equations describing the generating functions for each variable. 


\noindent \textbf{SCM. } Formally, a structural causal model~\cite{pearl2009book} is defined by a tuple $\langle \vecV, \vecU, \vecF, P(\vecu) \rangle$ where $\vecV$ is the set of observed variables, $\vecU$ refer to the unobserved variables, $\vecF$ is a set of functions, and $P(\vecU)$ is a strictly positive probability measure for $\vecU$. For each  $V\in \vecV$, $f_V \in \vecF$ determines its data-generating process, 
$V=f_V(\operatorname{Pa}_V, U_V)$ where $\operatorname{Pa}_V \subseteq \vecV \setminus \{V\}$ denotes \textit{parents} of $V$ and $U_V\subseteq \vecU$. We consider a non-linear, additive noise SCM such that $\forall V \in \vecV$, $f_V$ can be written as a additive combination of some $f^*_V(\operatorname{Pa}_V)$ and the unobserved variables (error terms).  We assume a Markovian SCM such that  unobserved variables (corresponding to error terms) are mutually independent, thus the SCM corresponds to a directed acyclic graph (DAG) over $\vecV$ with edges to each node from its parents. Note that a specific realization of the unobserved variables, $\vecU=\vecu$  determines the values of all other variables.

\noindent \textbf{Counterfactual. }Given an SCM, values of unobserved variables $\vecU=\vecu$,  a target variable $Y\in \vecV$ and a subset of inputs $\vecX \subseteq \vecV\setminus \{Y\}$,  a counterfactual corresponds to the query, \textit{``What would have been the value of $Y$ (under $\vecu$), had $\vecX$ been $\vecx$}''. It is written as $Y_{\vecx}(\vecu)$.

Using counterfactuals, we can formally express the attribution question in the the above example. Suppose the observed values are $Y=y_t$ and $X_i=x_i$ for some input $X_i$, under $\vecU=\vecu$. At an earlier reference timestamp with a different value of the unobserved variables, $\vecU=\vecu'$, the values are $Y=y'$ and $X_i=x'_i$.  Starting from the observed value ($\vecU=u$), the attribution for $X_i$ is characterized by the change in $Y$ after changing $X_i$ to its reference value,  $Y_{x_i}(\vecu) - Y_{x'_i}(\vecu)= y_t - Y_{x'_i}(\vecu)$. 
That is, given that $Y$ is $y_t$ with $X_i=x_i$ and all other variables at their observed value, how much would $Y$ change if $X_i$ is set to $x'_i$? 
Similarly, we can ask, $Y_{x_i, x'_1}(\vecu) - Y_{x'_i, x'_1}(\vecu)$ ($i \neq 1$), denoting the change in $Y$'s value upon setting $X=x_i$ when $X_1$ is set to its reference values. Thus, there can be multiple expressions to determine the counterfactual impact of $X_i$ depends on the values of other variables.


\section{Attribution using CF-Shapley value}
To develop an attribution score, we propose a way to average over the different possible counterfactual impacts. First, we posit desirable axioms that an attribution score should satisfy, as in~\cite{lundberg2017unified,pmlr-v162-jung22a}.
\subsection{Desirable axioms for an attribution score} \label{sec:axioms}
\begin{axiom*} Given two values of the metric, observed, $Y(\vecu)$ and reference, $Y(\vecu')$, corresponding to unobserved variables,  $\vecu$ and $\vecu'$ respectively, following properties are desirable for an attribution score $\bm{\phi}$ that measures the causal contribution of inputs $V \in \vecV$.
\begin{enumerate}
    \item \textbf{CF-Efficiency. } The sum of attribution scores for all $V \in \vecV$ equals the counterfactual change in output from reference to observed value, $Y(\vecu) -Y_{\vecv'}(\vecu)=Y_{\vecv}(\vecu') -Y(\vecu')=\sum_V \phi_V$.
    \item \textbf{CF-Irrelevance}. If a variable $X$ has no effect on the counterfactual value of output under all witnesses, $Y_{x',\vecs'}(\vecu)=Y_{\vecs'}(\vecu) \forall \vecS \subseteq \vecV \setminus \{X\}$, then $\phi_X=0$.
    \item \textbf{CF-Symmetry. } If two variables have the same effect on counterfactual value of output $Y_{\vecs'}(\vecu)-Y_{x'_1,\vecs'}(\vecu)=Y_{\vecs'}(\vecu)-Y_{x'_2,\vecs'}(\vecu) \forall \vecS \subseteq \vecV \setminus \{X_1, X_2\}$, then their attribution scores are same, $\phi_{X_1}=\phi_{X_2}$. 
    \item \textbf{CF-Approximation.} For any subset of variables $\vecS \subseteq \vecV$ set to their reference values $\vecs'$, the sum of attribution scores approximates the counterfactual change from observed value. I.e., there exists a weight $\omega(\vecS)$ s.t. the vector $\bm{\phi}_{\vecS}$ is the solution to the weighted least squares,  $\arg \min_{\bm{\phi}^*_{\vecS}} \sum_{\vecS \subseteq \vecV} \omega(\vecS) ((Y(\vecu) - Y_{\vecs'}(\vecu)) -\sum_{S \in \vecS} \phi^*_{S})^2$.   
\end{enumerate}
\end{axiom*}
Similar to shapley value axioms, these axioms convey intuitive properties that a \textit{counterfactual} attribution score should satisfy. \textit{CF-Efficiency} states the sum of attribution scores for inputs should equal the difference between the observed metric and the counterfactual metric when all inputs are set to their reference values. \textit{CF-Irrelevance} states that if changing the value of an input $X$  has no effect on the output counterfactual under all values of other variables, then the Shapley value of $X$ should be zero. \textit{CF-Symmetry }states that if changing the value of two inputs has the same effect on the counterfactual output under all values of the other variables, then both variables should have an identical attribution score. And finally, \textit{CF-Approximation} states the difference between the observed output and the counterfactual output due to a change in any subset of variables is roughly equal to the sum of attribution scores for those variables.

Note that \textit{CF-Efficiency} does not necessarily imply that the sum of attribution scores is equal to the  actual difference between the observed value and reference value. This is because the actual difference is a combination of the input variables' contribution and statistical noise (error terms). That is,  $y_t - y' = Y_{\vecv}(\vecu) - Y_{\vecv'}(\vecu') = \sum_V \phi_V + (Y_{\vecv'}(\vecu)- Y_{\vecv'}(\vecu'))$,  where we used the CF-Efficiency property for a desirable attribution score $\bm{\phi}$. The second term corresponds to the difference in metric with the same input variables but different noise corresponding to the observed and reference timestamps. 
This is the unavoidable noise component since we are explaining the change due to a \textit{single} observation.
Therefore, for any counterfactual attribution score to meaningfully explain the observed difference,  it is useful to select a reference timestamp to minimize the difference over exogenous factors (e.g., using a previous value of the metric on the same day of week or same hour). Given the true structural equations and an attribution score that satisfies the axioms, if the scores do sum to the observed difference in a metric, then it implies that reference timestamp was well-selected.

\subsection{The CF-Shapley value}
We now define the \actcause value that satisfies all four axioms.
\begin{definition}
Given an observed output metric $Y=y_t$ and a reference value $y'$, the CF-Shapley value for contribution by input $X$ is given by, 
\begin{equation} \label{eq:cf-shap}
     \phi_X = \sum_{\vecS \subseteq \vecV \setminus \{X\}}\frac{Y_{\vecs'}(\vecu)- Y_{x',\vecs'}(\vecu) }{n C(n-1, |\vecS|)}
\end{equation}
where $n$ is the number of input variables $\vecV$, $\vecS$ is the subset of variables set to their reference values $\vecs'$,  and $\vecU=\vecu$ is the value of unobserved variables such that $Y(\vecu)=y_t$.
\end{definition}
\begin{proposition}
\actcause value satisfies all four axioms, Efficiency, Irrelevance, Symmetry and Approximation.
\end{proposition}
\begin{proof}
\textit{Efficiency.} Following ~\cite{pmlr-v162-jung22a,vstrumbelj2014explaining}, the \actcause value for an input $V_i$ can be written as, 
\begin{equation}
    \phi_{V_i}= \frac{1}{n!} \sum_{\pi \in \Pi(n)} Y_{\vecw'_{pre}(\pi)}(\vecu) - Y_{v'_i, \vecw'_{pre}(\pi)}(\vecu) 
\end{equation}
where $\Pi$ is the set of all permutations over the $n$ variables and $\vecW_{pre}(\pi)$ is the subset of variables that precede $V_i$ in the permutation $\pi \in \Pi$. The sum is, 
\begin{equation}
    \begin{split}
        &\sum_{i=1}^{n} \phi_{V_i} = \frac{1}{n!} \sum_{\pi \in \Pi(n)} \sum_{i=0}^{n} Y_{\vecw'_{pre}(\pi)}(\vecu) - Y_{v'_i, \vecw'_{pre}(\pi)}(\vecu) \\
        &=\frac{1}{n!} \sum_{\pi \in \Pi(n)}  Y_{\emptyset}(\vecu) - Y_{\vecv'}(\vecu) \\
        &=  Y(\vecu) - Y_{\vecv'}(\vecu)
    \end{split}
\end{equation}
We can show it analogously under $\vecU=\vecu'$. 
\\
\textit{CF-Irrelevance.}
If $Y_{x', \vecs'}(\vecu) = Y_{\vecs'}(\vecu) \forall \vecS \subseteq \vecV \setminus \{X\}$, then the numerator in Eqn.~\ref{eq:cf-shap} for $\phi_X$, will be zero and the result follows. 
\\
\textit{CF-Symmetry.}
Assuming same effect on counterfactual value, we write the \actcause value for $V_i$ and show it is the same for $V_j$. 
\begin{equation*}
\begin{split}
    &\phi_{V_i} = \sum_{\vecW\subseteq \vecV \setminus \{V_i\}}\frac{Y_{\vecw'}(\vecu)- Y_{v'_i,\vecw'}(\vecu) }{n C(n-1, |\vecW|)}\\
            &= \sum_{\vecW\subseteq \vecV \setminus \{V_i, V_j\}}\frac{Y_{\vecw'}(u)- Y_{v'_i,\vecw'}(u)}{n C(n-1, |\vecW|)} + \sum_{\vecZ\subseteq \vecV \setminus \{V_i, V_j\}} \frac{Y_{v'_j,\vecz'}(u)- Y_{v'_i,v'_j,\vecz'}(u)}{n C(n-1, |\vecZ|+1)} \\ 
            &= \sum_{\vecW\subseteq \vecV \setminus \{V_i, V_j\}}\frac{Y_{\vecw'}(\vecu)- Y_{v'_j,\vecw'}(\vecu)}{n C(n-1, |\vecW|)} + \sum_{\vecZ\subseteq \vecV \setminus \{V_i, V_j\}} \frac{Y_{v'_i,\vecz'}(\vecu)- Y_{v'_i,v'_j,\vecz'}(\vecu)}{n C(n-1, |\vecZ|+1)} \\
            &= \sum_{\vecW\subseteq \vecV \setminus \{V_j\}}\frac{Y_{\vecw'}(\vecu)- Y_{v'_j,\vecw'}(\vecu)}{n C(n-1, |\vecW|)} = \phi_{V_j}
\end{split}
\end{equation*}
where the third equality uses $Y_{v'_i,\vecs'}(\vecu)=Y_{v'_j,\vecs'}(\vecu) \text{\ } \forall \vecS \subseteq \vecV \setminus \{V_i, V_j\}$.
\\
\textit{CF-Approximation.} Here we use a property~\cite{lundberg2017unified} on value functions of standard Shapley values.  There exists specific weights $\omega(S)$ such that the Shapley value is the solution to $\arg \min_{\bm{\phi}^*_{\vecS}} \sum_{\vecS\subseteq \vecV}\omega(\vecw) (\nu(S) -\sum_{\vecs \in \vecS} \phi^*_w)^2$ where $\nu(\vecS)$ is the value function of any subset $\vecS \subseteq \vecV$. The result follows by selecting $\nu(\vecS)=Y(\vecu)-Y_{\vecs'}(\vecu)$.  
\end{proof}


\noindent \textbf{Comparison to do-shapley. }
Unlike \actcause, the do-shapley value~\cite{pmlr-v162-jung22a} takes the expectation over all values of the unobserved $\vecu$,  $\bm{E}_{\vecu}[Y|do(\vecS)] - \bm{E}_{\vecu}[Y]$. Thus, it measures the \textit{average} causal effect  over values of $\vecu$, whereas for attributing a single observed value, we want to know the contributions of inputs under the same $\vecu$.

\subsection{Estimating \actcause values}
\label{sec:est-cf}
Eqn.~\ref{eq:cf-shap} requires estimation of counterfactual output at different (hypothetical) values of input, and in turn requires both the  causal graph and the structural equations of the SCM. Using knowledge on the system's metric computation, the first step is to  construct its computational graph. Then for each node in the graph, we fit its generating function using a predictive model over its parents, which we consider as the data-generating process (fitted SCM).

To fit the SCM equations, for each node $V$, a common way is to use supervised learning to build a model $\hat{f}_V$ estimating its value using the values of its parent nodes at the same timestamp. 
However, such a model will have high variance due to natural temporal variation in the node's value over time. Since including variables predictive of the outcome  reduces the variance of an estimate in general~\cite{bottou2013counterfactual}, we utilize auto-correlation in time-series data to include the previous values of the node as predictive features. Thus, the final model is expressed as, $\forall V \in \vecV$,
\begin{equation}\label{eq:catden-model}
    \hat{v}_t = \hat{f}(\operatorname{Pa}_V, v_{t-1},v_{t-2} \cdots , v_{t-r})
\end{equation}
where $r$ is the number of auto-correlated features  that we include. 
The model can be trained using a supervised time-series prediction algorithm with auxiliary features, such as DeepAR~\cite{salinas2020deepar}.

We then use the fitted SCM equations to estimate the counterfactual with the 3-step algorithm from 
Pearl~\cite{pearl2009book}, assuming additive error. 
To compute $Y_{\vecs'}(\vecu)$ for any subset $\vecS \subseteq \vecV$, the three steps are,
\begin{enumerate}
    \item \textbf{Abduction.} Infer error of structural equations on all observed variables. For each $V \in \vecV$,  $\hat{\epsilon}_{v,t} = v_t - \hat{f}_V(\operatorname{Pa}(V), v_{t-1}..v_{t-r})$ where $v_t$ is the observed value at timestamp $t$. 
    \item \textbf{Action.} Set the value of $\vecS\leftarrow s'$, ignoring any parents of $\vecS$.
    \item \textbf{Prediction. } Use the inferred error term and new value of $s'$ to estimate the new outcome, by proceeding step-wise for each level of the graph~\cite{pearl2009book,dash2022evaluating} (i.e., following a topological sort of the graph),  starting with $\vecS$'s children and proceeding downstream until $Y$ node's value is obtained.  For each $X \in \vecV$ ordered by the topological sort of the graph (after $\vecS$), $x'=\hat{f}_X(Pa'(X), \cdots) +\hat{\epsilon}_{x,t}$. And finally, we will obtain, $y'=\hat{f}_Y(Pa'(Y), \cdots) +\hat{\epsilon}_{y,t}$. 
\end{enumerate}
 Thus, the \actcause score for any input is obtained by repeatedly applying the above algorithm and aggregating the required counterfactuals in Eqn.~\ref{eq:cf-shap}; we use a common Monte Carlo approximation to sample a fixed number ($M=1000$) of values of $\vecS$~\cite{castro2009polynomial,fatima2008shapleycompute}. 

\section{Evaluation}
Our goal is to attribute observed changes in the output metric of an ad matching system. We first describe the system and conduct a simulation study to evaluate \actcause scores.

\subsection{Description of the ad matching system}
We consider an ad matching system where the goal is to retrieve all the relevant ads for a particular web search query by a user (these ads are ranked later to show only top ones to the user). The outcome variable is the average number of ads matched for each query, called the \textit{``matching density''} (or simply \textit{density}). This outcome can be affected by multiple factors, including the availability of ads by advertisers, the distribution and amount of user queries issued on the system, any algorithm changes, or any other system bug or unknown factors. For simplicity, we consider a  matching algorithm based on matching \textit{exact} full text between a query and provided keyword phrases for an ad. This algorithm remains stable over time due to its simplicity. Thus, we can safely assume that there are no algorithm changes or code bugs for the matching algorithm under study. Given an extreme or unexpected value of  density, our goal then is to attribute between change in ads and change in queries. 

Since there are millions  of queries and ads, we categorize the data by nearly 250 semantic query categories. Examples of query categories are "Fashion Apparel", "Health Devices",  "Internet", and so on.
A naive solution may be to simply compare the magnitude of observed change in ad demand or query volume across categories. That is, given a change in density on day $t$, choose a reference day $r$ (e.g., same day  last week) and compare the values of \ad and \qv. We may conclude that the factor with the highest percentage change is causing the overall change in density. However, the limitation is that 
the factor with the highest percentage change may neither be necessary nor sufficient to cause the change because its effect depends on the values of other factors. E.g.,  an increase in \qv for a category can either have positive, negative, or no effect on the daily density depending on its ad demand compared to other categories. This is because the density is computed as a query volume-weighted average of category density; increase in \qv for a low-demand (and hence low-density) category \textit{decreases} the aggregate density (see Eqn.~\ref{eq:dailyden-eqn}). 

\subsection{Constructing an SCM for ad density metric}
To apply the \actcause method  for attributing a matching density value, we define a causal graph based on how the metric is computed, as shown in 
Figure~\ref{fig:causal-graph}. The number of queries for a category  is measured by the number of search result page views (SRPV). The number of ads is measured by the number of listings posted by advertisers. For simplicity, we call them  \textit{\qv} and \textit{\ad}. We assume that given a category, the \ad and \qv are independent of each other since they are driven by the advertiser and user goals respectively.  The combination of \ad and \qv for a category determine its  category-wise density which then is aggregated to yield the \textit{daily density}.  As we are interested in attribution over days as a time unit, we refer to the aggregate density as daily density, $y$. 
Thus, the variables $\{ad^{c1}, qv^{c1},  ad^{c2}, qv^{c2} \cdots  ad^{ck}, qv^{ck} \}$ are the $2k$ inputs to the system where $ci$ is the category, $\dem$ refers to \ad, $\vol$ refers to \qv, and  $k$ is the number of categories. 

The structural equation from category-wise densities to daily density is known. It is simply a weighted average of the category-wise densities, weighted by the \qv.
\begin{equation}\label{eq:dailyden-eqn}
    y_t = f(\den^{c1}_t, \vol^{c1}_t, \cdots \den^{ck}_t, \vol^{ck}_t) = \frac{\sum_c \den^c_t \vol^c_t}{\sum_c \vol^c_t}
\end{equation}
where $den^c_t$ is the density of category $c$ on day $t$ and $\vol^c_t$ is the  \qv for the category on day $t$.
However, the equation from category-wise \ad and \qv to category density is  infeasible to obtain. This would involve ``replay'' of a computationally expensive matching algorithm to real-time queries and ad listings but the ad listings are not available (only a daily snapshot of ads inventory is stored in the logs). We will show how to to estimate the structural equation for category density in Section~\ref{sec:fit-scm}.


\subsection{Evaluating \actcause on simulated data}
Before applying \actcause on the ad matching system, we first evaluate the method on simulated data motivated by the causal graph of the system. This is because it is impossible to know the ground-truth attribution using data from a real-world system,  since we do not know how the change in input variables led to  the observed metric value and which inputs were the most important. 

We construct simulated data based on the causal structure of Figure~\ref{fig:causal-graph}. For each category, we assume \ad and \qv as independent Guassian random variables (we simulate real-world variation in \qv using a Beta prior). The category-wise density is constructed as a monotonic function of ad demand and has a biweekly periodicity. The SCM equations are,
\begin{align} \label{eq:sim-gen}
    \gamma &= \mathcal{B}(0.5,0.5); \text{\ }   \vol^c_t = \mathcal{N}(1000\gamma, 100); \text{\ }  \dem^c_t     = \mathcal{N}(10000, 100) \nonumber \\
        den^c_t   &= g(\dem^c_t,\vol^c_t, den^c_{t-1}) + \mathcal{N}(0,\sigma^2) \nonumber \\
                &= \kappa * \dem^c_t/\vol^c_t + \beta * a* den^c_{t-1} + \mathcal{N}(0,\sigma^2) \\
        y_t &= \frac{\sum_c den^c_t \vol^c_t}{\sum_c \vol^c_t} \label{eq:dailyden}
\end{align}
where $\vol^c_t$ and $\dem^c_t$ are the query volume and ad demand respectively for category $c$ at time $t$. They combine to produce the ad matching density $den^c_t$ based on a function $g$ and additive normal noise. The variance of the noise, $\sigma^2$ determines the stochastic variation in the system. For the simulation, we construct $g$ based on two observations about the category density: \textbf{1)} it is roughly a ratio of the \textit{relevant} ads and the number of queries; and \textbf{2)} it exhibits auto-correlation with its previous value and periodicity over a longer time duration. We use $\kappa$ to denote the fraction of relevant ads and add a second term with parameter $a$ to simulate a biweekly pattern, $a=1 \text{\ } \operatorname{if} \operatorname{floor}(t/7) \text{ is even} \operatorname{else} \text{ }a=-1 $.
$\beta$ is the relative importance of the previous value in determining the current category density.  
  Finally, all the category-wise densities are weighted by their query volume $\vol^c_t$ and averaged to produce the daily density metric, $y_t$.
  
Each dataset generated using these equations has 1000 days and 10 categories; we set $\kappa=0.85, \beta=0.15$ for simplicity. We intervene on the \ad or \qv of the 1000th point to construct an outlier metric that needs to be attributed. Given the biweekly pattern, reference date is chosen 14 days before the 1000th point.

\noindent \textbf{Setting ground-truth attribution.} Even with simulated data, setting ground-truth attribution can be tricky. For example, if there is an increase in \ad for one category and increase in \qv for another, it is not clear which one would cause the biggest impact on the daily density. That depends on their \qv and \ad respectively and any changes in other categories. To evaluate attribution methods, therefore, we consider simple interventions where objective ground-truth can be obtained. Specifically, for ease of interpretation, we intervene on  only two  categories at a time such that the first has a substantially higher chance of affecting the outcome metric than the second. 

We consider two configurations: change in 1) \ad and 2) \qv. For changing \ad (\textit{Config 1}), we choose two categories such that the first has the highest \qv and the second has the lowest \qv. We double the \ad for both categories with a slight difference (x2 for the first category, x2.1 for the second). Since the category-wise densities are weighted by \qv to obtain the daily density metric, for the same (or similar) change in demand, it is natural that first category has higher impact on daily density (even though they may have similar impact on their category-wise density). For \textit{Config 1}, thus, the ground-truth attribution is the first category. For changing \qv (\textit{Config 2}), we choose two categories such that the first has the most extreme density and the second has density equal to the reference daily density. Then, we change \qv as above: x2 for the first category and x2.1 for the second. Following Eqn.~\ref{eq:dailyden}, \qv change in a category having the same density as the daily density is expected to have low impact on daily density (keeping other categories constant, if category density is not impacted by \qv,  an increase in \qv for a category with density equal to daily density causes zero change in daily density). Thus, the ground-truth attribution (category with the highest impact on output metric) is again the first category. Note that \qv has higher variation across categories, so a higher multiplicative factor does not necessarily mean a higher absolute difference. 

\noindent \textbf{Baseline attribution methods.} We compare \actcause to the standard Shapley value (as implemented in SHAP~\cite{lundberg2017unified}, \stdshapley) and the  do-shapley value (\doshapley)~\cite{pmlr-v162-jung22a}. The \stdshapley method ignores the structure and fits a model directly predicting daily density $y_t$ using (category-wise) \ad and \qv features. It uses the predictions of this model for computing the Shapley score. For the \doshapley method, we notice that our causal graph corresponds to the Direct-Causal graph structure in their paper and use the estimator from Eq. (5) in ~\cite{pmlr-v162-jung22a}, that depends on the same daily density predictor as the standard Shapley value. 
We also evaluate on  three intuitive baselines based on absolute change in inputs: \textbf{1)} The category with the biggest change in \ad (\demdelta); \textbf{2)}  \qv (\supdelta); or \textbf{3)} density multiplied by \qv (\proddelta) since this product is used in the daily density equation. 

For the \actcause algorithm, we  fit the structural equation for category density, using the following features: \ad, \qv, $den_{t-1}, den_{t-7}, den_{t-14}$. For both the \actcause category density prediction and the \stdshapley daily density prediction model, we use a 3-layer feed forward network. We use all data uptil 999th day for training and validation for all models.  

\begin{figure}[tb]
    \centering
    \includegraphics[width=0.85\columnwidth]{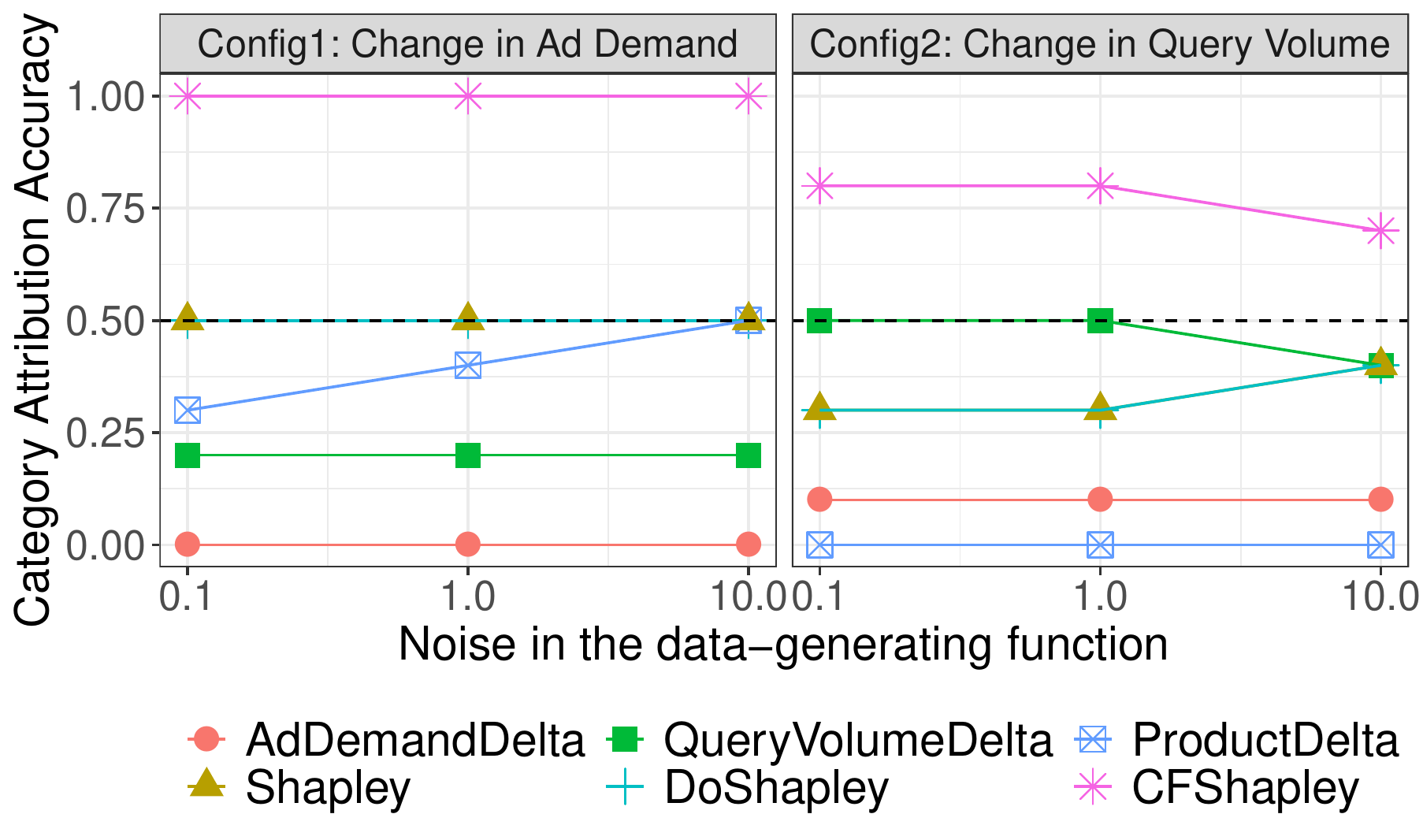}
    \caption{Category attribution accuracy for various methods.}
    \label{fig:sim-cat-attr}
\end{figure}

\noindent \textbf{Results. }
For each attribution method, we measure accuracy compared to the ground-truth as we increase the noise ($\sigma$) in the true data-generating process (SCM) ($\sigma=\{0.1, 1, 10\}$). As noise in the generating process increases, we expect higher error for fitting structural equations and thus the attribution task becomes harder.
\textit{Attribution accuracy} is defined as the fraction of times a method outputs the highest attribution score for the correct category (first category), over 20 simulations. 

Figure~\ref{fig:sim-cat-attr} shows the results. \actcause obtains the highest attribution accuracy for both Config 1 and 2. In general, attribution for \ad is easier than \qv because both the category density and daily density are monotonic functions of the \ad. That is why we observe near 100\% accuracy for \actcause under \textit{Config 1}, even with high noise. The attribution accuracy for \textit{Config 2} is 70-80\%, decreasing as more noise is introduced.

In comparison, none of the baselines achieve more than 50\% (random-guess) accuracy. Note that the \stdshapley and \doshapley methods obtain similar attribution accuracies. While their attribution scores are different, the highest ranked category often turns out to be the same since they rely on the same daily density model (but use different formulae). Inspecting the predictive accuracy of the daily density model offers an explanation: error on the daily density prediction is higher than that for category-wise density prediction (and it increases as the noise is increased). This indicates the value of computing an individualized counterfactual using the full graph structure, rather than focusing on the average (causal) effect. Finally, the other intuitive baselines fail on both tasks since they only look at the change in the input variables. 

\section{Case study on ad matching system}
We now apply the \actcause attribution method on data logs of a real-world ad matching system  from July 6 to Dec 28, 2021. For each query, we have log data  on the number of ads matched by the system. In addition, each query is marked with its category. The category \qv is measured as the number of queries issued by users for each category. This allows us to calculate the ground-truth matching density on each day, category-wise and aggregate. Separately, to compute the category-wise \textit{input} ad demand for a day, we fetch each ad listing available on the day and assign it to a category if any query from that category contains a word that is present in its keywords.   This is the total sum of ad listings that are potentially relevant to the query for the exact matching algorithm (that matches the full query exactly to the full ad keyword phrase). 

\subsection{Implementing \actcause: Fitting the SCM}
\label{sec:fit-scm}
We follow the method outlined in Section~\ref{sec:est-cf}. The main task is to estimate the structural equations for category-wise ad density.
There are over 250 categories; fitting a separate model for each is not efficient. Besides, it may be beneficial to exploit the common patterns in the  different time-series. We therefore consider a deep learning-based model, DeepAR~\cite{salinas2020deepar} that fits a single recurrent network for multiple timeseries (we also tried a transformer-based model, temporal fusion transformer (TFT)~\cite{lim2021temporal} but found it hard to tune to obtain comparable accuracy). As specified in Equation~\ref{eq:catden-model}, for each category, the DeepAR model is given \ad,  \qv and the autoregressive values of density for the past 14 days. Note that rather than predicting over a range of days (which can be innacurate), we fit the  timeseries model separately for each day using data up to its $t-1$th day, to utilize the additional information available from the previous day. To implement DeepAR, we used the  open-source GluonTS library. 

We compare the DeepAR model to three baselines. As simple baselines that capture the weekly pattern, we consider, \textbf{1)} category density on the same day a week before; and \textbf{2)} the average density over the last four weeks. We also consider a 3-layer feed-forward network that uses the same features as DeepAR.  Table~\ref{tab:deepar-acc} shows the prediction error. DeepAR model obtains the lowest error on the validation set according to all three metrics: mean absolute percentage error (MAPE), median APE, and the symmetric MAPE~\cite{makridakis2020m4}. 
For our results, we choose DeepAR as the estimated SCM equation and apply \actcause on data from Nov 15 to Dec 28. We chose Nov 15 to allow sufficient days of training data.

\noindent \textbf{Choosing reference timestamp. }
The \actcause method requires specifying a  reference day that provides the 
the ``expected/usual'' density value. Common ways to choose it are the  last day’s value or the value last week on the same day. We choose the latter due to known weekly patterns in the density metric. 

\begin{table}[tb]
    \centering
    \resizebox{0.7\columnwidth}{!}{%
    \begin{tabular}{lccc}
        \toprule
        Model       &Mean APE (\%)   & Median APE (\%) & sMAPE\\
        \midrule
        LastWeek    & 21.2     &  11.5      & 0.20\\
        Avg4Weeks   & 25.1     & 10.6       & 0.17 \\
        FeedForward & 20.0     & 10.8       & 0.20\\
        DeepAR      & \textbf{15.6}     & \textbf{7.8}        & \textbf{0.13} \\
        \bottomrule
    \end{tabular}}
    \caption{Accuracy of category-wise density prediction models.} 
    \label{tab:deepar-acc}
    \vspace{-3em}
\end{table}

\begin{figure}[tb]
    \centering
    \includegraphics[width=0.7\columnwidth]{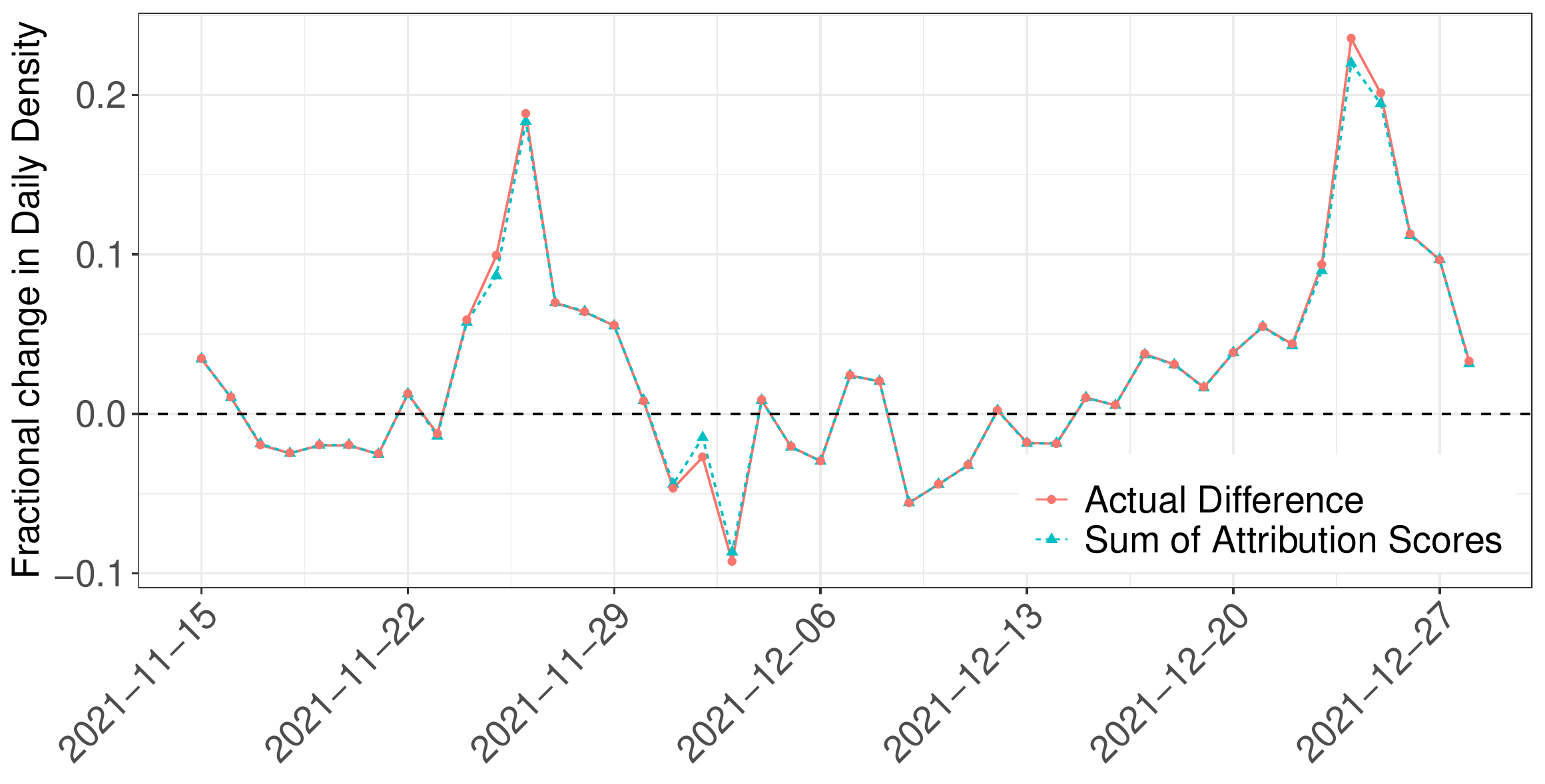}
    \caption{Comparison of actual percentage change in daily density  with sum of estimated \actcause attribution scores. }
    \label{fig:shapley-valid}
\end{figure}


\subsection{Validating the CF-Efficiency axiom}
 We first check whether the obtained \actcause scores sum up to the observed percentage change in daily density metric (Figure~\ref{fig:shapley-valid}). The difference between the sum of \actcause scores and the actual change is less than 0.10\% for all days, indicating that our choice of reference timestamp is appropriate (Sec.~\ref{sec:axioms}) and that the shapley value computation by approximation is capturing relevant signal. 

\subsection{Choosing dates for evaluation}
While we computed attribution scores for all days, typically one is interested in attribution for unexpected values for daily density. 

To discover unusual days for attribution, we fit a standard time-series model to the aggregate daily density data. 
We use four candidate models: \textbf{1)} daily density on the same day last week; \textbf{2)} mean density of the last 4 weeks; \textbf{3)} a feed forward network; and \textbf{4)} DeepAR model. As for the category-wise prediction, all neural network models are provided the last 14 days of daily density. Table\ref{tab:agg-acc} shows the mean APE, median APE, and SMAPE. The feed forward model obtains the lowest error. While DeepAR is a more expressive model than FeedForward, a potential reason for its lower accuracy is the number of training samples (only as many data points as the number of days for dailydensity prediction unlike category-wise prediction).  
For its simplicity, we use the FeedForward network for detecting outlier days. Its prediction for different days and the outliers detected can be seen in Figure~\ref{fig:ff-preds}. Like DeepAR, the feedforward model is implemented as a Bayesian probabilistic  model, so it outputs prediction samples rather than  a point prediction. 

\begin{table}[tb]
    \centering
    \resizebox{0.7\columnwidth}{!}{%
    \begin{tabular}{lccc}
        \toprule
        Model       & Mean APE (\%)   & Median APE (\%) & sMAPE \\
        \midrule
        LastWeek    & 4.6    & 3.1 & 0.047\\
        Last4Weeks  & 4.5    & 3.3  & 0.047\\
        FeedForward & \textbf{3.0}    & \textbf{2.2} & \textbf{0.031} \\
        DeepAR      & 3.4    & 2.4 & 0.035\\
        
        \bottomrule
    \end{tabular}}
    \caption{Prediction error for daily density models.}
    \label{tab:agg-acc}
    \vspace{-2em}
\end{table}

\begin{figure}[tb]
    \centering
    \includegraphics[width=0.65\columnwidth]{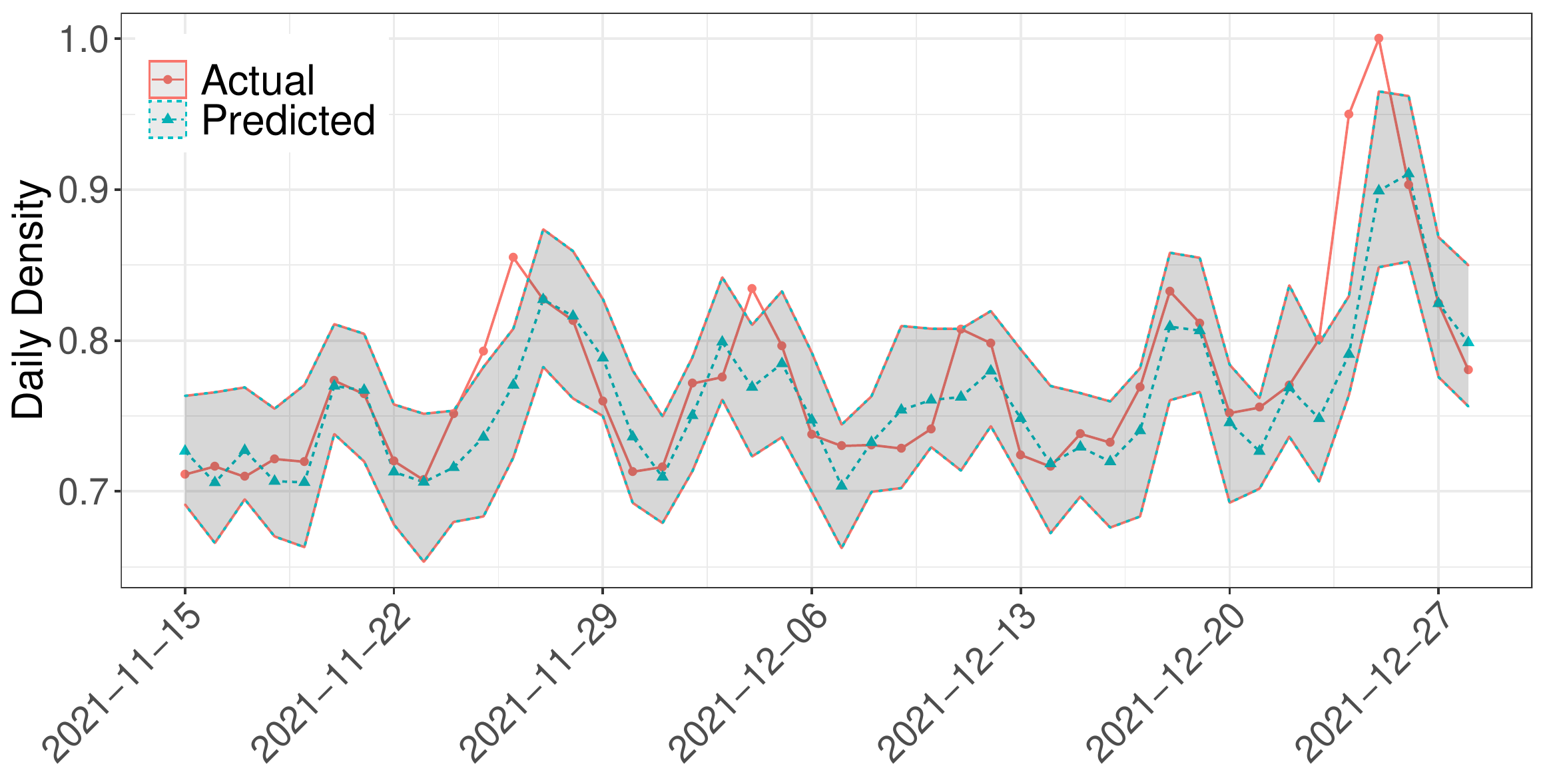}
    \caption{Outliers through FeedForward model's prediction. Shaded region represents the 95\% prediction interval.}
    \label{fig:ff-preds}
\end{figure}

\begin{figure}[tb]
    \centering
    \includegraphics[width=0.8\columnwidth]{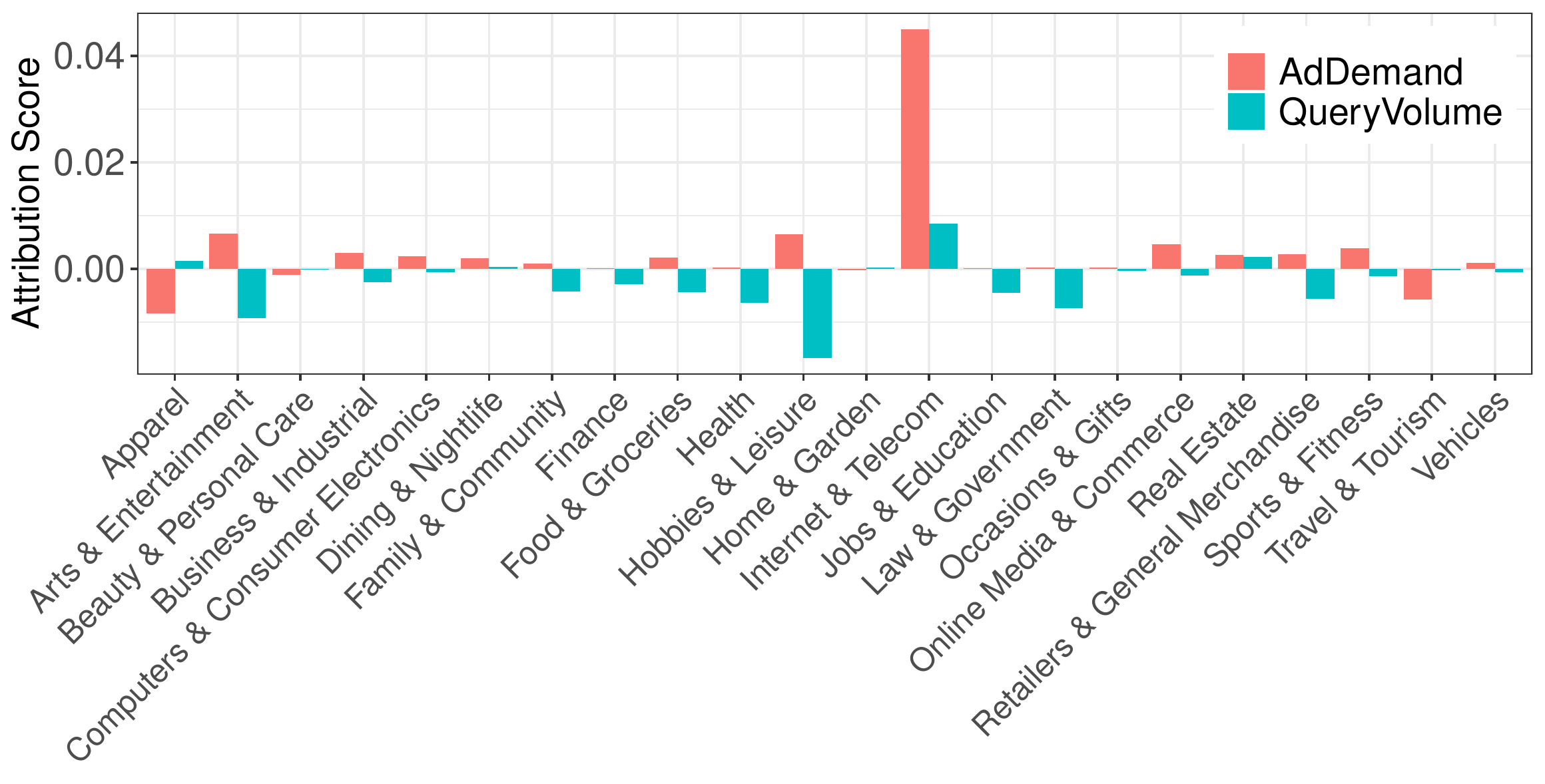}
    \caption{Attribution scores of different categories by \ad and \qv on December 4.}
    \label{fig:cat-contrib}
\end{figure}

\begin{table}
    \centering
    \resizebox{0.85\columnwidth}{!}{%
    \begin{tabular}{lcc}
    \toprule
        Category                 &       AdDemandAttrib    &   QueryVolumeAttrib \\
    \midrule
        \textit{Sort by AdDemandAttrib} & &\\
        Internet \& Telecom     &       \textbf{0.0450}      &   0.00850 \\
        Apparel                 &       -0.00843    &   0.00151 \\
        Arts \& Entertainment   &       0.00663     &   -0.00928 \\
        Hobbies \& Leisure      &       0.00646     &   -0.0168  \\
        Travel \& Tourism       &       -0.00584    &   -0.000287 \\
    \midrule
    \textit{Sort by QueryVolumeAttrib} & &\\
      Hobbies \& Leisure        &       0.00646     &   \textbf{-0.0168}  \\
      Arts \& Entertainment      &       0.00633     &   -0.00928 \\
      Internet \& Telecom       &       0.0450      &    0.00850 \\ 
      Law \& Government          &       0.000203    &   -0.00743 \\
      Health                      &       0.000161     &   -0.00645 \\
    \bottomrule
    \end{tabular}}
    \caption{Ad demand and \qv attribution on Dec 4. }
    \label{tab:dec4-attr}
    \vspace{-3em}
\end{table}

Days where the daily density goes beyond the 95\% prediction interval are chosen for attribution. A visual inspection shows two clusters,  Thanksgiving/Black Friday and Christmas, which are expected due to their significance  in the US. We also find an extreme value on Dec. 4. In all three cases, the daily density increases. Intuitively, one may have expected the opposite for holidays: density would decrease since people are expected to spend less time online.  

\subsection{Qualitative analysis}
We now use \actcause to explain these observed changes. 

\noindent \textbf{December 4. }
Figure~\ref{fig:cat-contrib} shows the attribution by different categories, aggregated up to obtain 22 high-level categories. The \textit{Internet \& Telecom} (IT) category has the biggest positive attribution score while the \textit{Hobbies \& Leisure} (HL) category has the biggest negative attribution score. That is, daily density decreased on the day due to the HL category. 

To understand why, we look at the attribution scores separately for \ad and \qv for each category in Table~\ref{tab:dec4-attr}. The attribution score reflects to the percentage change in daily density compared to last week, due to \ad or \qv of a category. The only categories to have an attribution score greater than 1\% are \textit{IT} and \textit{HL}, agreeing with the category-wise analysis. Specifically, the change in \ad due to  \textit{IT} leads to a 4.5\% increase in daily density. The \qv change in  \textit{HL}, on the other hand, leads to a 1.7\% decrease in daily density. Considering all categories together, \ad change  leads to an 6.5\% increase in daily density   and \qv change leads to a 5.6\% decrease. The net result is a 1\% improvement over the last week. While an increase of 1\% of daily density may look small, note that the value last week was already inflated due to it being a Black Friday week. This is why we detect outliers using the expected time-series pattern rather than simply difference from last week. On such days, one may also consider an alternative baseline, e.g., two weeks before.

Are the attributions meaningful? In the absence of ground-truth, we dive deeper into the query logs to check for evidence. We do find a significant increase in queries for the \textit{HL} category. In fact, more than 70\% of the increase in \qv for \textit{HL} is due to cheetah-related queries. On manual inspection, we find that December 4 is \textit{International Cheetah Day}. Cheetah-related queries also contribute to 86\% of the \ad increase for \textit{HL} category.  
Given that the category density of \textit{HL} is much lower than the daily density, this increase in \qv causes a \textit{decrease} in daily density, leading to the negative attribution score. Due to \ad volume increase (perhaps in anticipation of the Cheetah Day), the \textit{HL} also leads to an increase of  0.6\% in daily density (see Table~\ref{tab:dec4-attr}. On the other hand, \textit{IT} category's main contribution is from an increase in ad demand. Logs show a substantial (14\%) increase in ads compared to last week for the category on Dec 4, which explains its high attribution score for \ad. 
This increase is sustained across queries, possibly indicating a shift for the first Saturday after the holiday weekend.  

\noindent \textbf{Nov 25 and 26 (Thanksgiving).} On Thanksgiving holiday (Nov 25), we may have expected density to drop since many people in the US are expected to spend more time with their family and less time online. At the same time, online shopping on Black Friday (Nov 26) may increase density. Instead, we find that the density increases significantly on \textit{both} days (see Figure~\ref{fig:ff-preds}.   Specifically, compared to last week, daily density on Nov 26 increased by 18.3\%, out of which 13.5\% is contributed by \qv change and 4.8\% by \ad. How to explain this result? Using the \actcause method, for \qv change, we find that the categories \textit{Health}, \textit{Law and Government} and \textit{Business \& Industrial} are the top-ranked categories. Each contribute more than 2\% of the density increase, leading to a cumulative 7\% increase. From the logs, we see that \qv for these categories decreased as people spent less time on work or health related queries. Since these categories tend to have low density, the daily density increased as a result. On the \ad side, \textit{Online Media \& Ecommerce} contributed nearly 3\% increase in daily density, perhaps due to increased demand for Black Friday shopping. 
Nov 25 exhibits similar patterns for \qv. 

\noindent \textbf{Dec 24 and Dec 25 (Christmas).} On Christmas days too, there is an significant increase in density. Like the Thanksgiving days, health and work-related queries are issued fewer times, leading to an overall increase in daily density (all three categories have attribution scores >1\%). However, we find that the top categories by \qv change are \textit{Hobbies \& Leisure} and \textit{Arts \& Entertainment}. Both these categories experience a surge in their query volume and being high-density categories, cause a 2.1\% and 1.8\% increase in daily density respectively. To explain this, we look at the query logs and find that the rise in \textit{Hobbies \& Leisure} queries is fueled by the \textit{toys \& games} subcategory, which is aligned with the expectation of the holiday days. On Dec 25, \textit{Hobbies \& Leisure} is also the category which has the highest attribution score by \ad (2.7\%). 
Overall, the category contributes 4.8\% increase, nearly one-third of the total density increase on Christmas day, signifying the importance of \textit{toys \& games} subcategory for Christmas.

\section{Discussion and Conclusion}
We presented a counterfactual-based attribution method to explain changes in a large-scale ad system's output metric. Using the computational structure of the system, the method provides attribution scores that are more accurate than prior methods. 


\bibliographystyle{ACM-Reference-Format}
\bibliography{attribution}

\end{document}